\documentclass[11pt,a4paper]{article}
\usepackage{authblk}
\usepackage{graphicx}
\usepackage{amsmath}
\usepackage{amssymb}
\usepackage{algorithm}
\usepackage{algorithmic}
\usepackage{amsthm}
\usepackage{epstopdf}
\usepackage{caption}
\usepackage
{url}
\usepackage{subcaption}
\usepackage{natbib}
\usepackage{cases}
\usepackage{bm}
\newtheorem{theorem}{Theorem}

\newtheorem{proposition}{Proposition}

\theoremstyle{definition}
\newtheorem{assumption}{Assumption}
\newtheorem{definition}{Definition}
\title{\bf Nonlinear Distributional Gradient Temporal-Difference Learning    }
\date{}
\author[1]{Chao Qu}
\author[1]{Shie Mannor}
\author[2]{Huan Xu}
\affil[1]{ Faculty of Electrical Engineering,  Technion}
\affil[2]{H. Milton Stewart School of Industrial and Systems Engineering, Georgia Institute of Technology}
\begin{document}

\maketitle

\begin{abstract}
	We devise a distributional  variant of gradient temporal-difference (TD) learning.   Distributional reinforcement learning has been demonstrated to outperform the regular one in the recent study \citep{bellemare2017distributional}. In the policy evaluation setting, we  design two new algorithms called  distributional GTD2  and distributional TDC  using the Cram{\'e}r distance on the distributional version of the Bellman error objective function, which inherits  advantages of both the nonlinear gradient TD algorithms and the distributional RL approach. In the control setting, we propose the distributional Greedy-GQ using  the similar derivation. We prove the asymptotic almost-sure convergence of distributional GTD2 and TDC to a local optimal solution for   general smooth function approximators, which includes neural networks that have been widely used in recent study to solve the real-life RL problems. In each step, the computational complexities of above three algorithms are linear w.r.t.\ the number of the parameters of the function approximator, thus can be implemented efficiently for neural networks.
\end{abstract}

\section{Introduction}

Reinforcement learning (RL) considers a problem where an agent interacts with the environment to maximize the cumulative reward trough time. A standard approach   to solve the RL problem is called value function based reinforcement learning, which finds a policy that maximizes the value function $V(s)$ \citep{sutton1998reinforcement}.   Thus, the  estimation of  the value function of a given stationary policy of a Markov Decision Process (MDP) is an important subroutine of {\em generalized policy iteration} \citep{sutton1998reinforcement} and a key intermediate step to generate good control policy \citep{gelly2008achieving,tesauro1992practical}.  The value function is known to solve the Bellman equation, which succinctly describes the recursive relation on state-action value function $Q(s,a)$.
$$ Q^{\pi}(s,a)= \mathbb{E} {R(s,a)} + \gamma\mathbb{E}_{s',a'} Q^{\pi}(s',a'), $$
where the expectation is taken over the next state $s'\sim P(\cdot| s,a)$, the reward $R(s,a)$ and the action $a'$ from policy $\pi$, $\gamma$ is the discount factor. Hence, many RL algorithms are based on the idea of solving the above Bellman equation in a sample driven way, and one popular technique is the temporal-difference (TD) learning \citep{sutton1998reinforcement}. 

The last several years have witnessed the success of the TD learning  with the value function approximation \citep{mnih2015human,van2016deep}, especially when using a deep neural network. In their seminal work, \citet{tsitsiklis1996analysis} proved that the TD($\lambda$) algorithm converges when a {\em linear function approximator} is implemented and  states are sampled according to the policy evaluated 
(sometimes referred as {\em on-policy setting} in RL literature).  However, if either the  function approximator is non-linear, or the on-policy setting does not hold, there are counterexamples that demonstrates that TD($\lambda$) may diverge. To mitigate this problem, a family of TD-style algorithms called Gradient Temporal Difference (GTD) are proposed by \citep{sutton2009convergent,sutton2009fast} that address the instability  of the TD algorithm with the linear function approximator in the off-policy setting.  These works rely on the objective function called mean-squared projected Bellman error (MSPBE) whose unique optimum are the fixed points of the TD(0) algorithm.  \citet{bhatnagar2009convergent} extend this idea to the non-linear smooth function approximator (e.g.,  neural networks) and prove the convergence of the algorithm under mild conditions. In the control setting, \citet{maei2010toward} propose Greedy-GQ which has similar objective function as MSPBE but w.r.t. the Bellman optimality operator.

Recently, the distributional perspective on reinforcement learning has gained much attention. Rather than study on the expectation of the long term return (i.e., $Q(s,a)$), it explicitly takes into consideration the stochastic nature of the  long term return $Z(s,a)$ (whose expectation is $Q(s,a)$).  The recursion of $Z(s,a)$ is described by the distributional Bellman equation as follows,
\begin{equation}\label{distributional_Bellman}
Z(s,a)\overset{D}{=} R(s,a)+\gamma Z(s',a'),
\end{equation}
where $\overset{D}{=}$ stands for ``equal in distribution'' (see Section~\ref{sec.setting}   for more detailed explanations). The distributional Bellman equation essentially asserts that the distribution of $Z$ is characterized by  the reward $R$, the next random state-action $(s',a')$ following policy $\pi$ and its random return $Z(s',a')$. Following the notion in \citep{bellemare2017distributional} we call $Z$ the {\em value distribution}.
\citet{bellemare2017distributional}  showed that for a fixed policy the Bellman operator over value distributions is a $\gamma$-contraction in {\em a maximal form of the Wasserstein metric}, thus making it possible to learn the value distribution in a sample driven way. There are several advantages to study the value distributions: First, real-life decision makers sometimes are interested in seeking big wins on rare occasions or avoiding a small chance of suffering a large loss. For example, in financial engineering, this risk-sensitive scenario is one of the central topics. Because of this, risk-sensitive RL has been an active research field in RL \citep{heger1994consideration,defourny2008risk,bagnell2003covariant,tamar2016learning}, and the value distribution obviously provides a very useful tool in designing risk-sensitive RL algorithms. Second, it can model the uncertainty. \citet{engel2005reinforcement} leveraged the distributional Bellman equation to define a Gaussian process over the unknown value function. Third, from the algorithmic view, as discussed in \citep{bellemare2017distributional}, the distributional Bellman operator preserves multimodality in value distribution, which leads to more stable learning. From the  exploration-exploitation tradeoff perspective, if the value distribution is known, the agent can explore the region with high uncertainty, which is often called ``optimism in the face of uncertainty'' \citep{kearns2002near,o2017uncertainty}.

\mbox{\textbf{Contributions}:} Although  distributional approaches on RL (e.g., C51 in \cite{bellemare2017distributional}) have shown promising results,   theoretical properties of them are not well understood yet,  especially when the function approximator is nonlinear. As  nonlinear function approximation is inevitable if we hope to combine RL with  deep neural networks~--~a paradigm with tremendous recent success that enables automatic feature engineering and end-to-end learning to solve the real problem.  Therefore, to extend the applicability of the distributional approach to the real problem and close the gap between the theory and practical algorithms, we propose the nonlinear distributional gradient temporal-difference learning.  It inherits the merits of non-linear gradient TD and distributional approaches  mentioned above. Using the similar heuristic, we also propose a distributional control algorithm called distributional Greedy-GQ.


The contributions of this paper are the following.
\begin{itemize}
	\item  We propose a {\em distributional MSPBE (D-MSPBE) } as the objective function to optimize, which is an extension of MSPBE when the stochastic nature of the random return is considered. 
	\item We derive two stochastic gradient algorithms to optimize the D-MSPBE using the weight-duplication trick in \citep{sutton2009convergent,bhatnagar2009convergent}. In each step, the computational complexity is linear w.r.t.\ the number of parameters of the function approximator, thus can be efficiently implemented for neural networks. 	
	\item We propose a distributional RL algorithm in the control setting called distributional Greedy-GQ, which is an distributional counterpart of \cite{maei2010toward}.
	\item We prove distributional GTD2 and TDC converge to a local optimal solution in the policy evaluation setting under mild conditions using the two-time-scale stochastic approximation argument. If the linear function approximator is applied we have the finite sample bound.

\end{itemize}
Remarks: More precisely, we have $m$ addition operations in each step of algorithm but the costs of them are negligible compared to computations in neural networks in general. Thus  the computational complexity in each step is still linear w.r.t.\ the number of parameters of the function approximator (neural networks).

\section{Problem setting and preliminaries}\label{sec.setting}  

We consider a standard setting in the reinforcement learning, where an agent acts in a stochastic environment by sequentially choosing actions over a sequence of time steps, in order to maximize the cumulative reward \citep{sutton1998reinforcement}. This problem is formulated as a Markov Decision Process (MDP) which is a 5-tuple ($\mathcal{S}, \mathcal{A}, \mathcal{R}, \mathcal{P}, \gamma$): $\mathcal{S}$ is the finite state space, $\mathcal{A}$ is the finite action space, $\mathcal{P}=(P(s'|s,a))_{s,s'\in \mathcal{S},a\in \mathcal{A}}$ are the transition probabilities, $R=(R(s,a))_{s,s'\in \mathcal{S},a\in \mathcal{A}}$ are the real-valued immediate rewards and $\gamma\in (0,1)$ is the discount factor.  A policy is used to select actions in the MDP. In general, the policy is stochastic and denoted by $\pi$, where $\pi(s_t,a_t)$ is the conditional probability density at $a_t$ associated with the policy.  We also define $R^{\pi} (s)=\sum_{a\in \mathcal{A}}\pi(s,a) R(s,a)$, $P^{\pi}(s,s')=\sum_{a\in \mathcal{A}} \pi(s,a) P(s'|s,a)$.

Suppose the policy $\pi$ to be evaluated is followed and it generates a trajectory $(s_0,a_0,r_1,s_1,a_1,r_2,s_2,...)$. We are given an infinite sequence of 3-tuples $(s_t,r_t,s'_t)$ that satisfies the following assumption.
\begin{assumption}\label{Assumption1}
	$ (s_t)_{t\geq 0}$ is an $\mathcal{S}$-valued stationary Markov process, $s_t\sim d(\cdot)$, $r_t=R^{\pi}(s_t)$ and $s'_t\sim P^{\pi}(s_t,\cdot)$.
\end{assumption}

Here $d(\cdot)$ denotes the probability distribution over initial states for a transition. Since stationarity is assumed, we can drop the index $t$ in the $t^{th}$ transition $ (s_t,r_t,s'_t)$ and use $(s,r,s')$ to denote a random transition.  The (state-action) value function $Q^{\pi}$ of a policy $\pi$ describes the expected return from taking action $a\in \mathcal{A}$ from state $s\in \mathcal{S}$.
$$ Q^{\pi} (s,a):= \mathbb{E} {\sum_{t=0}^{\infty} \gamma^t R(s_t,a_t) },$$
$ s_t\sim P(\cdot|s_{t-1},a_{t-1}), a_{t}\sim \pi(\cdot|s_t), s_0=s, a_0=a.  $
It is well known that the value function $Q^{\pi}(s,a)$ satisfies the Bellman equation.
$ Q^{\pi}(s,a)=\mathbb{E}R(s,a)+\gamma \mathbb{E} Q^{\pi} (s',a')$. Define the Bellman operator as $(\mathcal{T} Q^{\pi})(s,a):= \mathbb{E}R(s,a)+\gamma \mathbb{E}Q^{\pi}(s',a'),$ then the Bellman equation becomes $Q^{\pi}=\mathcal{T} Q^{\pi}$.  To lighten the notation, from now on we may drop the superscript $\pi$  when the policy to be evaluated is kept fixed.

\subsection{Distributional Bellman equation and Cram{\'e}r distance}
Recall that the return $Z$ is the sum of discounted reward along the agent's trajectory of interactions with the environment, and hence $Q (s,a)=\mathbb{E} Z (s,a)$. When the stochastic nature of the return $Z$ is considered, we need a distributional variant of the Bellman equation which the {\em distribution of $Z$} satisfies.   Following the notion in \citep{bellemare2017distributional}, we define the transition operator $P^{\pi}$:
$$P^{\pi} Z(s,a):\overset{D}{=}Z(s',a'), ~~~ s'\sim P(\cdot|s,a), a'\sim \pi(\cdot|s'),$$
where $A:\overset{D}{=} B$ indicates that the random variable $A$ is distributed according to the same law of $B$.
The distributional Bellman operator $\mathcal{T^{\pi}}$ is 
$ \mathcal{T}^{\pi} Z(s,a):\overset{D}{=} R(s,a)+\gamma P^{\pi} Z(s,a).$ For more rigorous definition and discussions on this operator, we refer reader to \citep{bellemare2017distributional}.

\citet{bellemare2017distributional} prove that the distributional Bellman operator is a $\gamma$-contraction in a maximal form of the Wasserstein metric. However as pointed by them (see proposition 5 in their paper), in practice, it is hard to estimate the Wasserstein distance using samples and furthermore the gradient estimation w.r.t.\ the parameter of the function approximator is biased in general. Thus KL divergence is implemented instead in the algorithm C51 rather than the Wasserstein metric. However the KL divergence may not be robust to the discrepancies in support of distribution \citep{arjovsky2017wasserstein}. In this paper, we adapt the Cram{\'e}r distance \citep{szekely2003statistics,bellemare2017cramer} instead, since the unbiased sample gradient estimaton of Cram{\'e}r distance can be easily obtained in the setting of reinforcement learning \citep{bellemare2017cramer}.
The square root of Cram{\'e}r distance is defined as follows: Suppose there are two distributions $P$ and $Q$ and their cumulative distribution functions are $F_{P}$ and $F_{Q}$ respectively, then the square root of Cram{\'e}r distance between $P$ and $Q$ is 
$$ \ell_2(P,Q):=\big(\int_{-\infty}^{\infty} (F_P(x)-F_Q(x))^2 dx \big)^{1/2}.$$
Intuitively, it can be thought as the two norm on the distribution function. Indeed, the distributional Bellman operator is a $\sqrt{\gamma}$-contraction in a maximal form of the square root of Cram{\'e}r distance. Here, for two random return $Z_1,Z_2 $ with distribution $P_{Z_1}$ and $P_{Z_2}$, a maximal form of the square root of Cram{\'e}r distance is defined as  $\bar{\ell}_2 (P_{Z_1}, P_{Z_2})=\sup_{s,a} \ell_2 (P_{Z_1}(s,a), P_{Z_2}(s,a)).$ 

\begin{proposition}\label{proposition}
	$\bar{\ell}_2$ is a metric over value distributions and $\mathcal{T}^{\pi}$ is a $\sqrt{\gamma}-$ contraction in $\bar{\ell}_2$.
\end{proposition}
The proof makes use of Theorem 2 in \citep{bellemare2017cramer}, and is deferred to appendix.

\subsection{ Gradient TD and Greedy-GQ}
We now review linear and nonlinear gradient TD and Greedy-GQ proposed by \cite{sutton2009convergent,bhatnagar2009convergent,maei2010toward}, which helps to   better understand the nonlinear distributional gradient TD and distributional Greedy-GQ. One approach in reinforcement learning for large scale problems  is to use a linear function approximation for the value function $V$. Particularly, the value function $\hat{V}(s)=\theta^T \phi(s)$, where the feature map is $\phi: \mathcal{S}\rightarrow \mathbb{R}^d$, and $\theta \in \mathbb{R}^d$ is the parameter of the linear model. The objective function of the gradient TD family is the mean squared projected Bellman error (MSPBE).
\begin{equation}\label{MSPBE}
MSPBE(\theta)=\frac{1}{2} \| \hat{V}-\Pi T \hat{V}  \|^2_D,
\end{equation}
where $\hat{V}$  is the vector of value function over $\mathcal{S}$, $D$ is a diagonal matrix with diagonal elements being the stationary distribution $d(s)$ over $\mathcal{S}$ induced by the policy $\pi$, and $\Pi$ is the weighted projection matrix onto the linear space spanned by $\phi(1),...,\phi(|\mathcal{S}|)$, which is 
$ \Pi= \Phi (\Phi^TD\Phi)^{-1} \Phi^T D$. Substitute $\Pi$ into \eqref{MSPBE}, the MSPBE can be written as
\begin{equation}\label{MSPBE1}
\begin{split}
MSPBE(\theta)=&\frac{1}{2} \| \Phi^T D ( \hat{V}-T \hat{V} ) \|^2_{ (\Phi^TD\Phi )^{-1} }\\
&= \frac{1}{2}\mathbb{E}[\delta \phi]^T \mathbb{E}[{\phi \phi^T}]^{-1} \mathbb{E}[\delta \phi],
\end{split}
\end{equation}
where $\delta$ is the TD error for a given transition $(s,r,s')$, i.e., $\delta=r+\gamma \theta^T\phi'-\theta^T\phi.$
Its negative gradient is $ -\frac{1}{2} \nabla MSPBE (\theta)=\mathbb{E}[(\phi-\gamma\phi')\phi^T ]w, $
where $w=\mathbb{E}[\phi \phi^T]^{-1}\mathbb{E}[\delta\phi]. $ \citet{sutton2009convergent} use the weight-duplication trick to update $w$ on a "faster" time scale as follows $ w_{t+1}=w_t+\beta_t (\delta_t-\phi_t^Tw_t)\phi_t.$ Two different ways to update $\theta$ leads to GTD2 and TDC.
\begin{equation}\label{equ:GTDTDC}
\begin{split}
&\theta_{t+1}=\theta_t+\alpha_t (\phi_t-\gamma \phi'_t)(\phi_t^Tw_t)~~~(GTD2),\\
&\theta_{t+1}=\theta_t+\alpha_t \delta_t \phi_t-\alpha_t\gamma \phi'_t(\phi_t^Tw_t)~~~(TDC).
\end{split}
\end{equation}

Once the nonlinear approximation is used, we can optimize a slightly different version of MSPBE. There is an additional term $h_t$ in the update rule 
$$\theta_{t+1}=\theta_t+\alpha_t \{(\phi_t-\gamma \phi'_t)(\phi_t^Tw_t)-h_t\}~~GTD2$$
See more discussion in section \ref{section:NDGTD}.

Similarly, Greedy-GQ optimizes following objective function,
$$J(\theta)= \|\Pi T^{\pi(\theta)} Q_{\theta}-Q_{\theta}\|_D^2,$$
where $\pi_{\theta}$ is a greedy policy w.r.t. $Q_{\theta}$. Reusing the weight-duplication trick, \citet{maei2010toward} give the update rule.
\begin{equation*}
\begin{split}
&\theta_{t+1}=\theta_t+\alpha_t[\delta_t \phi_t-\gamma (w_t^T\phi)\hat{\phi}],\\
&w_{t+1}=w_t+\beta_t[\delta_{t+1}-\phi^Tw_t]\phi_t,
\end{split}
\end{equation*}
where $\hat{\phi}$ is an unbiased estimate of expected value of the next state under $\pi_{\theta}.$

\section{Nonlinear Distributional Gradient TD}\label{section:NDGTD}In this section, we propose  distributional Gradient TD algorithms by considering the  Cram{\'e}r distance between value distribution of $Z(s,a)$ and $\mathcal{T} Z(s,a)$ which is a distributional counterpart of \cite{bhatnagar2009convergent}. To ease the exposition, in the following we consider the value distribution on state $s$ rather than the state-action pair $(s,a)$ since the extension to $(s,a)$ is straightforward.

\subsection{Distributional MSPBE (D-MSPBE)}
Suppose there are $|\mathcal{S}|=n$ states. One simple choice of the objective function is as follows 
\begin{equation}\label{equ:dmsbe}
\sum_{i=1}^{n} d(s_i) \ell_2^2 (Z(s_i),\mathcal{T}Z(s_i)).
\end{equation}
However, a major challenge to optimize \eqref{equ:dmsbe} is the double sampling problem, i.e.,  two independent samples are required from each state. To see that, notice that if we only consider the {\em expectation} of the return, \eqref{equ:dmsbe} 
reduces to the mean squared Bellman error (MSBE), and the corresponding algorithms to minimize MSBE are the well-known  residual gradient algorithms  \citep{baird1995residual}, which is known to require two independent samples for each state~\citep{dann2014policy}.  To get around the double sampling problem, we instead adapt MSPBE into its distributional version.  To simplify the problem, we follow \citep{bellemare2017distributional} and assume that the value distribution is discrete with range $[V_{\min}, V_{\max}]$
and whose support is the set of atoms $\{ z_j=V_{\min}+(j-1)\Delta z: 1\leq j\leq m \}$, $\Delta z:= \frac{V_{\max}-V_{\min}}{m-1}$. In practice $V_{\min}$ and $V_{\max}$ are not hard to get. For instance, suppose we know the bound on the reward $|r|<b$, then we can take $V_{\min}, V_{\max}$ as $\pm \frac{b}{1-\gamma}$. We further assume that the atom probability can be given by a parametric model $\theta$ such as a softmax function: 
$$ p_{\theta}(s_i, z_j)=\frac{\exp(l_\theta(s_i,z_j))}{\sum_{j=1}^{m} \exp(l_\theta(s_i,z_j))} ,$$
where $\ell_\theta(s_i,z_j)$ can be a non-linear function, e.g.,  a neural network. From an algorithmic perceptive, such assumption or approximation is necessary, since it is hard to represent the full space of probability distributions.

We denote the (cumulative) distribution function of $Z(s)$ as $F_\theta(s,z)$.  Notice $F_\theta (s,z)$ is non-linear  w.r.t.\ $\theta$ in general, thus it is not restricted to a hyperplane as that in the linear function approximation. Following the non-linear gradient TD in \citep{bhatnagar2009convergent}, we need to define  the tangent space at $\theta$. Particularly, we denote $F_\theta  \in \mathbb{R}^{nm\times1} $ as a vector of $F_\theta (s_i,z_j), i=1,...,n, j=1,...m$ and assume $F_\theta$ is a differentiable function of $\theta$. $\mathcal{M}=\{F_\theta \in \mathbb{R}^{nm\times1}| \theta \in \mathbb{R}^d\}$ becomes a differentiable submanifold of $\mathbb{R}^{mn\time 1}$. Define $  \phi_\theta(s,z)= \frac{\partial F_{\theta}(s,z)}{\partial \theta}  $ , then the tangent space at $\theta$ is $T \mathcal{M}_\theta=\{ {\Phi_\theta a}| a\in \mathbb{R}^d\}$, where $\Phi_\theta \in \mathbb{R}^{mn\times d} $ is defined as $(\Phi_\theta)_{((i,j),l)}= \frac{\partial}{\partial \theta_l} F_\theta(s_i,z_j)$, i.e., each row of it is $\phi^T_\theta (s_i,z_j)$. Let $\Pi_\theta$ be the projection that projects vectors to $T\mathcal{M}$. Particularly, to project the distribution function $F(s,z)$ onto the $\mathcal{T}M$ w.r.t. the Cram{\'e}r distance, we need to solve the following problem
\begin{equation}
\hat{\alpha}=\arg\min_\alpha \sum_{i=1}^{n}\sum_{j=1}^{m} d(s_i) \big( F (s_i,z_j)-\phi_\theta(s_i,z_j)^T \alpha   \big)^2,
\end{equation} 
where $F(s_i, z_j)$ is the value of distribution function of $Z(s_i)$ at $z_j$.
Since this is a least squares problem,  we have that the projection operator has a closed form $$\Pi_\theta=\Phi_\theta (\Phi_\theta^T D \Phi_\theta)^{-1} \Phi_\theta^T D,$$
where $D$ is a $nm\times nm$ diagonal matrix with diagonal elements being $ d(s_1) I^{m\times m},d(s_2) I^{m_\times m},..., d(s_n) I^{m_\times m}.$ Given this projection operator, we propose the distributional MSPBE (D-MSPBE). Particularly, the objective function to optimize is as follows 
\begin{equation*}
\underset{\theta}{\mbox{minimize:}}\quad \|F_{\theta}-\Pi G_\theta\|_D^2,
\end{equation*}
where $G_\theta \in \mathbb{R}^{nm\times 1} $ is the vector form of $G_\theta (s_i,z_j)$, and $G_\theta (s_i,z_j)$ is the value of distribution function of $\mathcal{T}Z(s_i)$ at atom $z_j$. Assume $(\Phi_\theta^T D \Phi_\theta)^{-1}$ is non-singular,  similar to the MSPBE, we rewrite the above formulation into another form.\\
{\mbox{\bf  D-MSPBE:}}
\begin{equation}\label{equ:DMSPBE}
\quad   \underset{\theta}{\mbox{minimize:}}\quad J(\theta):=\|\Phi_{\theta}^T D (F_\theta-G_{\theta})\|^2_{ (\Phi_{\theta}^T D \Phi_{\theta})^{-1}}.
\end{equation}
To better understand D-MSPBE, we compare it with 
MSPBE. First, in equation \eqref{MSPBE1}, we have the term $\hat{V}-T\hat{V} $, which is the difference between the value function $\hat{V}$ and $T\hat{V}$, while  we have the difference between the distribution of $Z$ and $TZ$ in equation \eqref{equ:DMSPBE}. Second, the $D$ matrix is slightly different, since in each state we need $m$ atoms to describe the value distribution. Thus we have the diagonal element as $d(s_i) I^{m\times m}$. Third, $\Phi_\theta$ is a gradient for $F_\theta$ and thus depends on the parameter $\theta$ rather than a constant feature matrix,  which is similar to \citep{bhatnagar2009convergent}.

\subsection{Distributional GTD2 and Distributional TDC}
In this section, we use the stochastic gradient to optimize the  D-MSPBE (equation \ref{equ:DMSPBE}) and derive the update rule of distributional GTD2 and distributional TDC.

The first step is to estimate $\Phi_\theta^T D (F_\theta-G_{\theta})$ from samples.  We denote $\hat{G}_{\theta}$ as the empirical distribution of $G_\theta$  and $\mathbb{E} \hat{G}_{\theta}=G_{\theta}$. Notice one unbiased empirical distribution $\hat{G}_\theta(s,\cdot)$ at state $s$ is the distribution of $ r+ \gamma Z(s')$, whose distribution function is $F_\theta(s', \frac{z-r}{\gamma})$ by simply shifting and shrinking the distribution of $Z(s')$. Then we have
\begin{equation*}
\begin{split}
\Phi_\theta^T D (F_{\theta}-G_{\theta})=&\mathbb{E} \sum_{j=1}^{m}\phi_\theta(s,z_j) \big( F_{\theta}(s,z_j)-\mathbb{E} \hat{G}_{\theta} (s,z_j)  \big)\\
= &\mathbb{E} \sum_{j=1}^{m}\phi_\theta(s,z_j) (F_\theta(s,z_j)-\hat{G}_\theta(s,z_j)).\\
\end{split}
\end{equation*}
Then we can write  D-MSPBE in the following way
\begin{equation*}
\begin{split}
&J(\theta)=\mathbb{E} \big(\sum_{j=1}^{m}\phi_\theta(s,z_j) (F_\theta(s,z_j)-\hat{G}_\theta(s,z_j))\big)^T\times\\ 
&\big( \mathbb{E} \sum_{j=1}^{m}\phi_\theta(s,z_j) \phi_\theta^T(s,z_j) \big)^{-1}  \big(\mathbb{E} \sum_{j=1}^{m}\phi_\theta(s,z_j)\times\\
&(F_\theta(s,z_j)-\hat{G}_\theta (s,z_j)) \big).
\end{split}
\end{equation*}

We define $\bm{\delta_\theta(s,z_j) }= \hat{G}_\theta (s,z_j)-F_{\theta} (s,z_j) $,  analogous to the temporal difference, and call it  {\em temporal distribution difference}.  To ease the exposition, we denote $A=\mathbb{E} \sum_{j=1}^{m}\phi_\theta(s,z_j) \phi_\theta^T(s,z_j).$   Then we have
\begin{equation}
\begin{split}
J(\theta)=\mathbb{E} \big( \sum_{j=1}^{m}\phi_{\theta}(s,z_j) &\delta_\theta (s,z_j) \big)^T A^{-1} \times\\ &\mathbb{E} \sum_{j=1}^{m}\big(\phi_{\theta}(s,z_j) \delta_\theta (s,z_j) \big).
\end{split}
\end{equation}
In the following theorem, we choose $\hat{G}_\theta(s,z)=F_\theta(s', \frac{z-r}{\gamma})$, an unbiased empirical distribution we mentioned above and give the gradient of D-MSPBE w.r.t. $\theta$. We defer the proof to the appendix.
\begin{theorem}\label{theorem:gradient}
	Assume that $F_\theta(s,z_j)$ is twice continuously differentiable in $\theta$ for any $s\in \mathcal{S}$ and $j\in \{1,...,m\}$, $\mathbb{E} \sum_{j=1}^{m}\phi_\theta(s,z_j) \phi_\theta^T(s,z_j)$ is non-singular in a small neighborhood of $\theta$. Denote $h=\mathbb{E} \sum_{j=1}^{m} (\delta_\theta (s,z_j)-w^T \phi_\theta(s,z_j))\nabla^2_{\theta} F_\theta (s,z_j)w$, then we have
	\begin{equation}\label{equ:DGTD_theta_theorem}
	\begin{split}
	-\frac{1}{2} \nabla_\theta J(\theta)= \mathbb{E} \sum_{j=1}^{m}\big(\phi_\theta(s,z_j)-&\phi_\theta(s',\frac{z_j-r}{\gamma} )\big)\times\\
	&\phi^T_\theta(s,z_j)w-h,	
	\end{split}
	\end{equation} 
	which has another form 	
	\begin{equation}\label{equ:DTDC_theta_theorem}
	\begin{split}
	-\frac{1}{2} \nabla_\theta J(\theta)=&-\mathbb{E} \sum_{j=1}^{m}  \phi_\theta(s',\frac{z_j-r}{\gamma} )\phi_{\theta}^T(s,z_j)w\\
	&+\mathbb{E} \sum_{j=1}^{m} (\phi_{\theta}(s,z_j)\delta_\theta(s,z_j))
	-h,
	\end{split}
	\end{equation}
	where $w=A^{-1}  \mathbb{E} \sum_{j=1}^{m}\big(\phi_{\theta}(s,z_j) \delta_\theta (s,z_j) \big). $
\end{theorem}

Based on Theorem~\ref{theorem:gradient}, we obtain the algorithm of distributional GTD2 and distributional TDC. Particularly   \eqref{equ:DGTD_theta_theorem} leads to Algorithm \ref{alg:DGTD2}, and \eqref{equ:DTDC_theta_theorem} leads to Algorithm \ref{alg:DTDC}. The difference  between distributional gradient TD methods and regular ones are highlighted in boldface.

\begin{algorithm}[h]
	\caption{Distributional GTD2 for policy evaluation }
	\label{alg:DGTD2}
	\begin{algorithmic}
		\STATE {\bfseries Input:} step size $\alpha_t$, step size  $\beta_t$, policy $\pi$.
		\FOR{$t=0,1,...$}	
		\STATE{
			\begin{equation*}
			w_{t+1}=w_{t}+\beta_t\bm{\sum_{j=1}^{m}} \big(-\phi_{\theta_t}^T(s_t,z_j)w_t+  \bm{\delta_{\theta_t} } \big) \phi_{\theta_t}(s_t,z_j).
			\end{equation*}	
			\begin{equation*}
			\begin{split}
			\theta_{t+1}=&\Gamma [   \theta_t+  \alpha_t \{\bm{\sum_{j=1}^{m}}\big(\phi_{\theta_t}(s_t,z_j)\\
			&-\bm{\phi_{\theta_t}(s_{t+1},\frac{z_j-r_t}{\gamma}} )\big) \phi^T_{\theta_t}(s_t,z_j)w_t-h_t\}].
			\end{split}
			\end{equation*}				
			$\Gamma: \mathbb{R}^d\rightarrow \mathbb{R}^d $ is a projection onto an compact set $C$ with a smooth boundary. }
		\STATE   $h_t= \bm{\sum_{j=1}^{m}} (\bm{\delta_{\theta_t}} -w_t^T \phi_{\theta_t}(s_t,z_j))\nabla^2 F_{\theta_t} (s_t,z_j)w_t,$
		\STATE where $\bm{\delta_{\theta_t}}= F_{\theta_t} (s_{t+1},\frac{z_j-r_t}{\gamma})-F_{\theta_t}(s_t,z_j). $
		\ENDFOR    
	\end{algorithmic}
\end{algorithm}

\begin{algorithm}[h]
	\caption{Distributional TDC for policy evaluation }
	\label{alg:DTDC}
	\begin{algorithmic}
		\STATE {\bfseries Input:} step size $\alpha_t$, step size  $\beta_t$, policy $\pi$.
		\FOR{$t=0,1,...$}	
		\STATE{	
			\begin{equation*}
			w_{t+1}=w_{t}+\beta_t\bm{\sum_{j=1}^{m}} \big(-\phi_{\theta_t}^T(s_t,z_j)w_t+ \bm{ \delta_{\theta_t}}  \big) \phi_{\theta_t}(s_t,z_j).
			\end{equation*} 
			\begin{equation*}
			\begin{split}
			\theta_{t+1}=&\Gamma [   \theta_t+  \alpha_t \{\bm{\sum_{j=1}^{m}} \big( \bm{\delta_{\theta_t}}\phi_{\theta_t}(s_t,z_j) -\\
			&\bm{\phi_{\theta_t}(s_{t+1},\frac{z_j-r_t}{\gamma})}(\phi_{\theta_t}^T(s_t,z_j) w_t)\big)   -h_t\}].
			\end{split}
			\end{equation*}
			
			$\Gamma: \mathbb{R}^d\rightarrow \mathbb{R}^d $ is a projection onto an compact set $C$ with a smooth boundary.}
		\STATE   $h_t= \bm{\sum_{j=1}^{m}} (\bm{\delta_{\theta_t}} -w_t^T \phi_{\theta_t}(s_t,z_j))\nabla^2 F_{\theta_t} (s_t,z_j)w_t,$
		\STATE where $\bm{\delta_{\theta_t}}= F_{\theta_t} (s_{t+1},\frac{z_j-r_t}{\gamma})-F_{\theta_t}(s_t,z_j). $
		\ENDFOR    
	\end{algorithmic}
\end{algorithm}

Some remarks are in order. We use distributional GTD2 as an example, but all remarks hold  for the distributional TDC as well.

(1). We stress the difference between the the update rule of GTD2 in \eqref{equ:GTDTDC} and that of the distributional GTD2 (highlighted in boldface): In the distributional GTD2, we use the temporal distribution difference $\bm{\delta_{\theta_t}}$ instead of the temporal difference in GTD2. Also notice there is a summation over $z_j$, which corresponds to the integral in the Cram{\'e}r distance, since we need the difference over the whole distribution rather than a certain point. The term $\frac{z_j-r_t}{\gamma}$ comes from the shifting and shrinkage on the distribution function of $Z(s_{t+1})$. 

(2). The term $h_t$ results from the nonlinear function approximation, which is zero in the linear case. This term is similar to the one in nonlinear GTD2 \citep{bhatnagar2009convergent}. Notice we do not need to explicitly calculate the Hessian in the term $\nabla^2 F_{\theta_t} (s,z)w$. This term can be evaluated using forward and backward propagation in neural networks with the complexity scaling linearly w.r.t.\ the number of parameters in neural networks, see the work \citep{pearlmutter1994fast} or chapter 5.4.6 in \citep{christopher2016pattern}.  We give an example in the appendix to illustrate how to calculate this term.

(3).  It is possible that $ \frac{z_j-r_t}{\gamma}$ is not on the support of the distribution in practice. Thus we need to approximate it by projecting it on the support of the distribution, e.g., round to the nearest atoms. This projection step would lead to further errors, which is out of scope of this paper. We refer  readers to  related discussion in \citep{dabney2017distributional,rowland2018analysis}, and leave its analysis as a future work.


(4). The aim of $w_t$ is to estimate $w$ for a fixed value of $\theta$. Thus $w$ is updated on a "faster" timescale and parameter $\theta$ is updated on a "slower" timescale.

\section{Distributional Greedy-GQ}

In practice, we care more about the control problem. Thus in this section, we propose the distributional Greedy-GQ for the control setting. 
Now we denote $F_{\theta} ((s,a),z)$ as the distribution function of $Z(s,a)$. Policy $\pi_{\theta}$ is a greedy policy w.r.t. $Q(s,a)$. i.e., the mean of $Z(s,a)$. $G_{\theta}((s,a),z)$ is the distribution function of $\mathcal{T}^{\pi_{\theta}}Z((s_i,a_i))$. The aim of the distributional Greedy-GQ is to optimize following objective function.
\begin{equation*}
\min_{\theta} \|F_{\theta}-\Pi G_{\theta}\|_D^2 
\end{equation*}
Using almost similar derivation as distributional GTD2 (With only difference in notation and we omit the term $h_t$ here), we give the following algorithm \ref{alg:distributional_GQ}, analogous to the Greedy-GQ \cite{maei2010toward}. 
\begin{algorithm}[h]
	\caption{Distributional Greedy-GQ }
	\label{alg:distributional_GQ}
	\begin{algorithmic}
		\STATE {\bfseries Input:} step size $\alpha_t$, step size  $\beta_t$, $0\leq\eta\leq 1$ 
		\FOR{$t=0,1,...$}	
		\STATE{	
			
			$Q(s_{t+1},a)=\sum_{j=1}^{m} z_j p(s_t,a),$ where $p(s_t,a)$ is the density function w.r.t. $F_\theta((s_t,a))$. $a^*=\arg\max_a Q(s_{t+1},a).$
			\begin{equation*}
			\begin{split}
			w_{t+1}=w_{t}+\beta_t&{\sum_{j=1}^{m}} \big(-\phi_{\theta_t}^T((s_t,a_t),z_j)w_t+
			{ \delta_{\theta_t}}  \big)\\ &\times\phi_{\theta_t}((s_t,a_t),z_j).
			\end{split}
			\end{equation*} 
			\begin{equation*}
			\begin{split}
			&\theta_{t+1}=  \theta_t+  \alpha_t \{{\sum_{j=1}^{m}} \big( {\delta_{\theta_t}}\phi_{\theta_t}((s_t,a_t),z_j) -\\
			&\eta{\phi_{\theta_t}((s_{t+1},a^*),\frac{z_j-r_t}{\gamma})}(\phi_{\theta_t}^T((s_t,a_t),z_j) w_t)\big) \}.
			\end{split}
			\end{equation*}
		}
		\STATE where ${\delta_{\theta_t}}= F_{\theta_t} ((s_{t+1},a^*),\frac{z_j-r_t}{\gamma})-F_{\theta_t}((s_t,a_t),z_j). $
		\ENDFOR    
	\end{algorithmic}
\end{algorithm}

Some remarks are in order.
\begin{itemize}
	\item $0\leq \eta\leq 1$ interpolates the distributional Q-learning and distributional Greedy-GQ. When $\eta=0$, it reduces to the distributional Q-learning with Cram{\'e}r distance while C51 uses KL-divergence. When $\eta=1$ it is the distributional Greedy-GQ where we mainly use the temporal distribution difference $\delta_{\theta_t}$ to replace the TD-error in \cite{maei2010toward}.
	\item Unfortunately for the nonlinear function approximator and control setting, so far we do not have convergence guarantee.  If the linear function approximation is used, we may obtain a asymptotic convergence result following the similar argument in \cite{maei2010toward}. We leave both of them as the future work.
\end{itemize}

\section{Convergence analysis}

In this section, we analyze the convergence of distributional GTD2 and distributional TDC and leave the convergence of distributional Greedy-GQ as an future work. To the best of our knowledge, the convergence of control algorithm even in the non-distributional setting is still a tricky problem. The argument  essentially follows the two time scale analysis  (e.g., theorem 2 in \citep{sutton2009fast}  and  \citep{bhatnagar2009convergent}). We first define some notions used in our theorem. Given  a compact set $C\subset \mathbb{R}^d$, let $\mathcal{C} (C)$ be the space of continuous mappings $C\mapsto \mathbb{R}^d$. Given projection $\Gamma$ onto $C$, let operator $\hat{\Gamma}: \mathcal{C} (C)\mapsto \mathcal{C}(\mathbb{R}^d)$ be 
$$ \hat{\Gamma} v(\theta)= \lim_{0<\epsilon\rightarrow 0} \frac{\Gamma(\theta+\epsilon v(\theta))-\theta}{\epsilon}. $$
When $\theta \in C^{\circ}$ (interior of $C$), $\hat{\Gamma} v(\theta)=v(\theta)$ . Otherwise, if $\theta \in \partial C$, $\hat{\Gamma} v(\theta)$ is the projection of $v(\theta)$ to the tangent space of $\partial C$ at $\theta
$. Consider the following ODE:
$$\overset{\cdot}{\theta}=\hat{\Gamma} (-\frac{1}{2} \nabla_\theta J ) (\theta), \theta(0) \in C,$$
where $J(\theta)$ is the D-MSPBE in \eqref{equ:DMSPBE}. Let $K$ be the set of all asymptotically stable equilibria of the above ODE. By the definitions $K\subset C$. We then have the following convergence theorem, which proof is deferred to the appendix.

\begin{theorem}\label{Theorem}
	Let $ (s_t,r_t,s'_t)_{t\geq0}$ be a sequence of transitions satisfying Assumption \ref{Assumption1}. The positive step-sizes in Algorithm \ref{alg:DGTD2} and \ref{alg:DTDC} satisfy $\sum_{t=0}^{\infty} a_t=\infty$, $\sum_{t=0}^{\infty}\beta_t=\infty$, $\sum_{t=0}^{\infty} \alpha_t^2,\sum_{t=1}^{\infty} \beta_t^2 <\infty$ and $\frac{\alpha_t}{\beta_t}\rightarrow 0,$ as $t\rightarrow \infty$  . Assume that for any $\theta \in C$ and $s\in \mathcal{S}$ s.t. $d(s)>0$, $F_\theta$ is three times continuously differentiable. Further assume that for each $\theta\in C$, $\big( \mathbb{E} \sum_{j=1}^{m}\phi_\theta(s,z_j) \phi_\theta^T(s,z_j) \big)$ is nonsingular. Then $\theta_t\rightarrow K$ in Algorithm \ref{alg:DGTD2} and \ref{alg:DTDC}, with probability one, as $t\rightarrow \infty.$
\end{theorem} 

If we assume the distribution function can be approximated by the \textit{linear function}. We can obtain a \textit{finite sample} bound. Due to the limit of space we defer it to appendix.

\section{Experimental result}

\subsection{Distributional GTD2 and distributional TDC}
In this section we assess the empirical performance of the proposed distributional GTD2 and distributional TDC and compare the performance with their non-distributional counterparts, namely, $GTD2$ and $TDC$.  Since it is hard to compare the performance of distributional GTD2 and TDC with their non-distributional counterparts in the pure policy evaluation environment, we use a simple control problem cart-pole problem to test the algorithm, where we do several policy evaluation steps to get a accurate estimation of value function and then do a policy improvement step. To apply distributional GTD2 or distributional TDC, we use a neural network to approximate the distribution function $F_{\theta} ((s,a),z)$. Particularly, we use a neural network with one hidden layer, the inputs of the neural network are state-action pairs, and the output is a softmax function. There are $50$ hidden units and we choose the number of atoms as $30$ in the distribution, i.e., the number of outputs in softmax function is $30$. In the policy evaluation step, the update rule of $w$ and $\theta$ is simple, since we just need to calculate the gradient of $F_{\theta}$, which can be obtained by the forward and backward propagation. The update rule of $h_t$ is slightly more involved, where we have the term $\nabla^2 F_{\theta} (s,z_j)w_t$. Roughly speaking, it requires four times as many   computations as the regular back propagation and we present the update rule in the appendix.  In the control step, we use the $\epsilon$-greedy policy over the expected action values, where $\epsilon$ starts at $0.1$ and decreases gradually to $0.02$.  To implement regular nonlinear GTD2 and TDC \citep{bhatnagar2009convergent}, we still use one hidden layer neural network with 30 hidden units. The output is $Q_\theta(s,a)$. The control policy is $\epsilon$ greedy with $\epsilon=0.1$ at the beginning and decreases to $\epsilon=0.02$ gradually.  In the experiment, we choose discount factor $\gamma=0.9$. Since  reward is bounded in $[0,1] $, in distributional GTD2 and distributional TDC, we choose $V_{\min}=0$ and $V_{\max}=10$. In the experiment, we use $20$ episodes to evaluate the policy, and then choose the policy by the $\epsilon$-greedy strategy.  We report experimental results (mean performance with   standard deviation) in the left and mid panel of Figure \ref{Fig:cart-pole}. All experimental results are averaged over 30 repetitions.   We observe that the distributional GTD2 has the best result,  followed by the distributional TDC. The distributional TDC seems to improve the policy faster at the early stage of the training and then slows down. The performance of regular GTD2 and TDC are inferior than their distributional counterparts. We also observe that standard deviations of the distributional version are smaller than those of regular one. In addition, the performance of the distributional algorithms increases steadily with few oscillations. Thus the simulation results show that the distributional RL is more stable than the regular one which matches the argument and observations in  \citep{bellemare2017distributional}. In the right panel, we draw a distribution function of $Z(s,a)$ estimate by the algorithm.

To test whether the algorithm converges in off-policy setting, we run experiment on grid world to compare distributional GTD2 (with atoms number=50) and GTD2 in the off-policy setting. The target policy is set to be the optimal policy. The data-generating policy is a perturbation on the optimal policy (0.05 probability to choose random action). We ran distributional GTD2 and GTD2 and calculate MSPBE every 1k timestep. Results are shown in figure \ref{Fig:MSPBE}. Both GTD2 and distributional GTD2 converge with similar speed while distributional GTD2 has smaller variance.  

\begin{figure*}
	\includegraphics[width=0.32\textwidth]{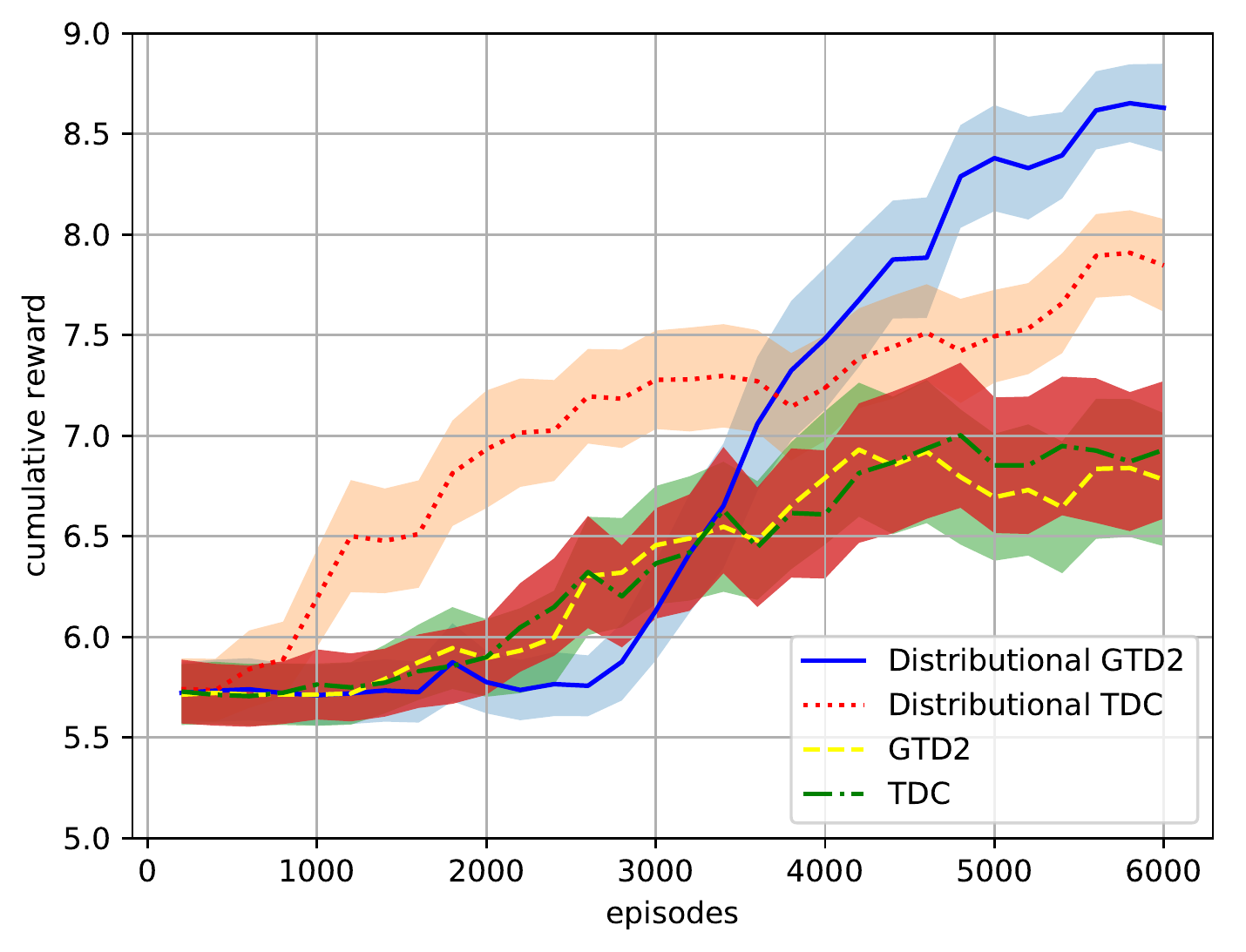}
	\includegraphics[width=0.32\textwidth]{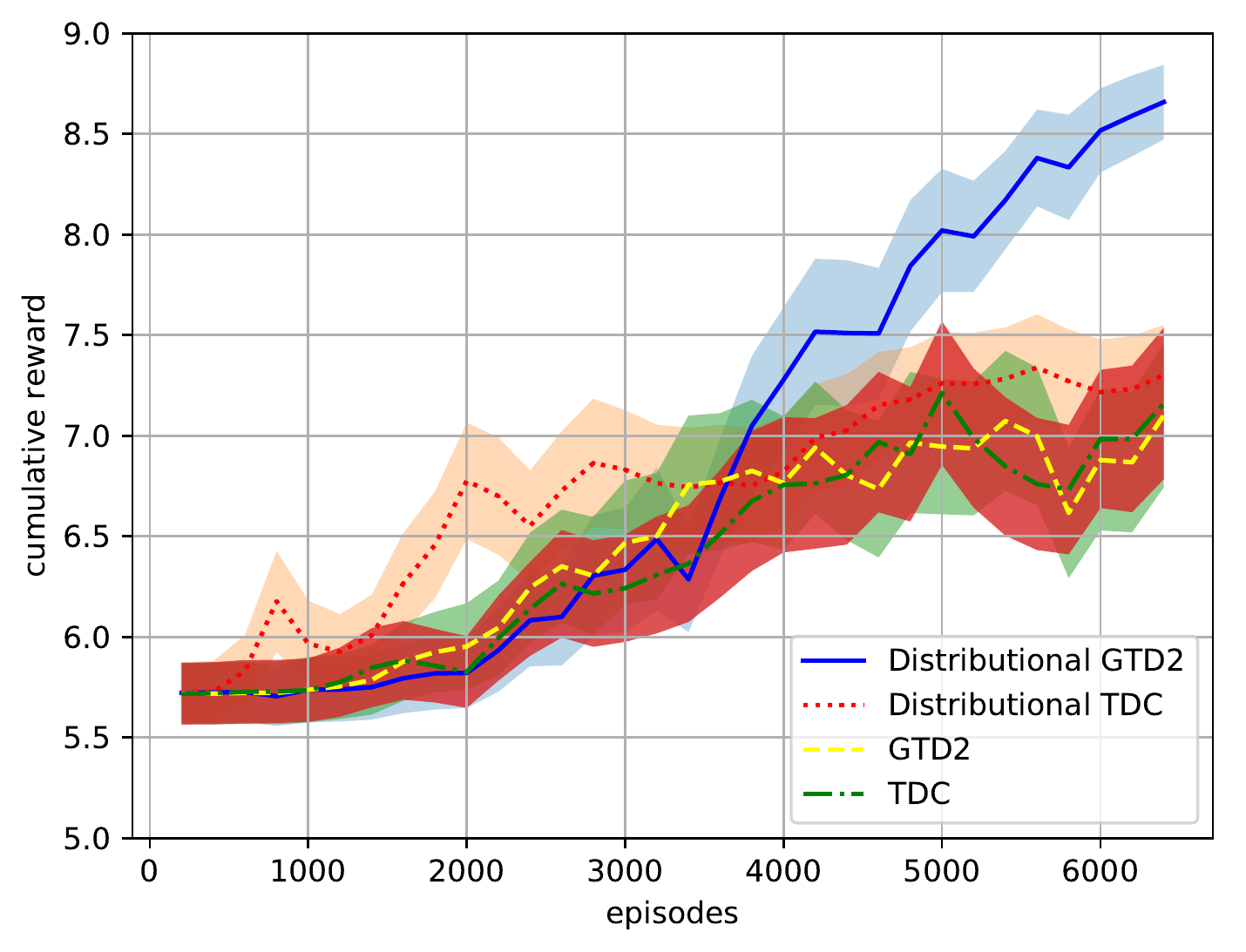}
	\includegraphics[width=0.32\textwidth]{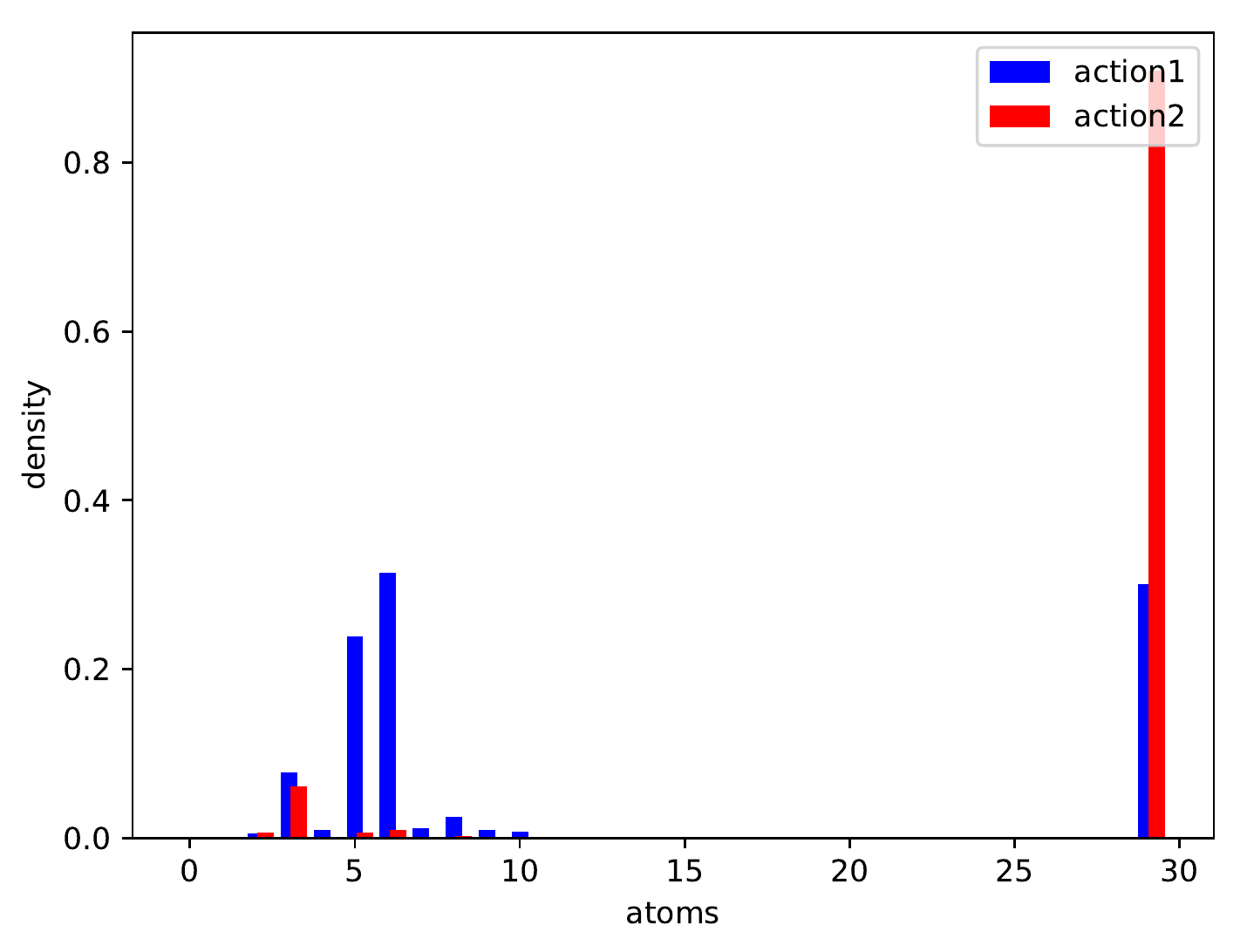}
	\caption{Left: Performance of algorithms in Cartpole v0. Middle: Performance of algorithms in Cartpole v1. Right: distribution of $Z(s,a)$ at some state $s$ }\label{Fig:cart-pole}
\end{figure*}

\begin{figure}
	\centering
	\includegraphics[width=0.35\textwidth]{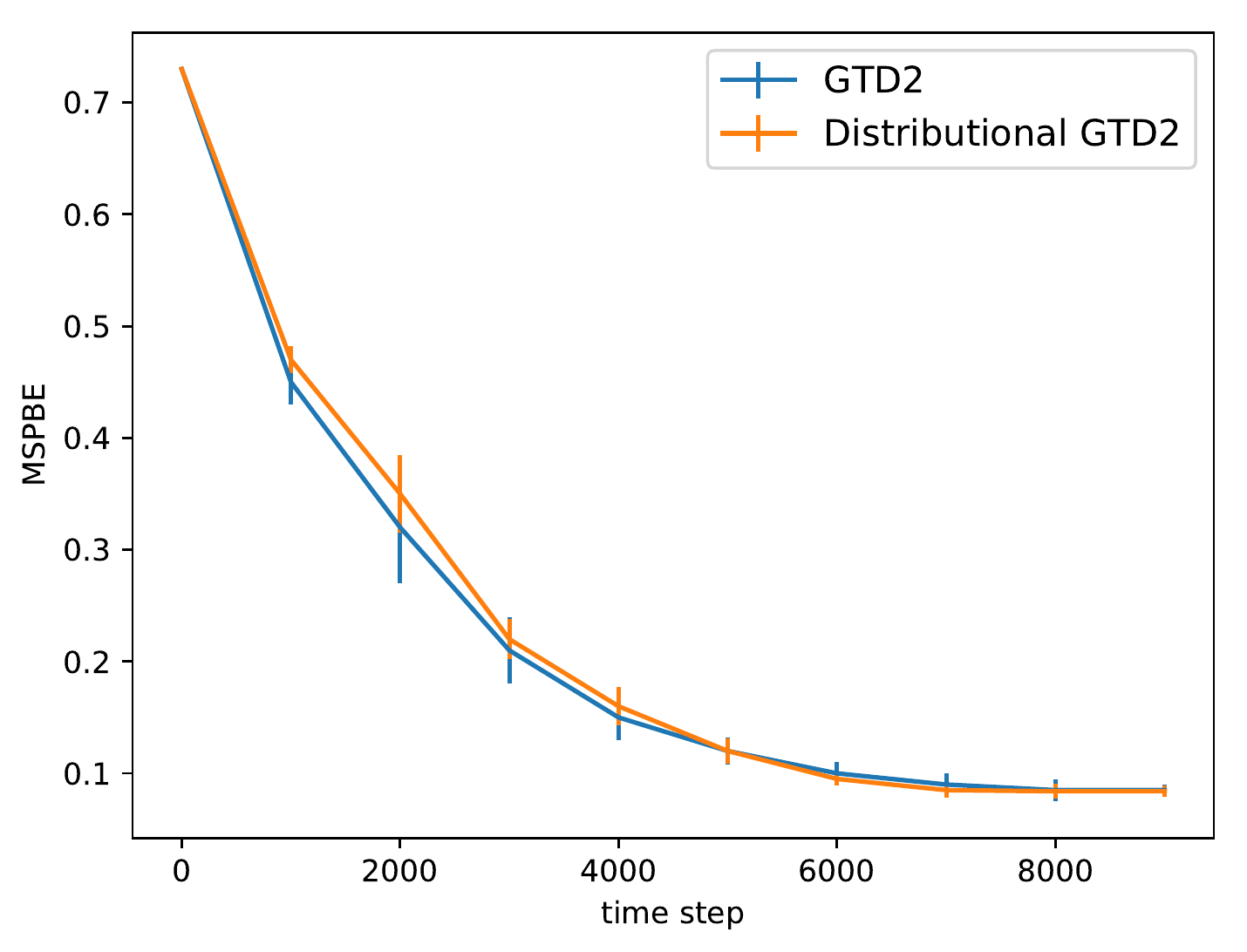}
	\caption{Compare GTD2 and distributional GTD2 in the off-policy setting with the measure of MSPBE.}\label{Fig:MSPBE}
\end{figure}
\begin{figure*}
	\includegraphics[width=0.32\textwidth]{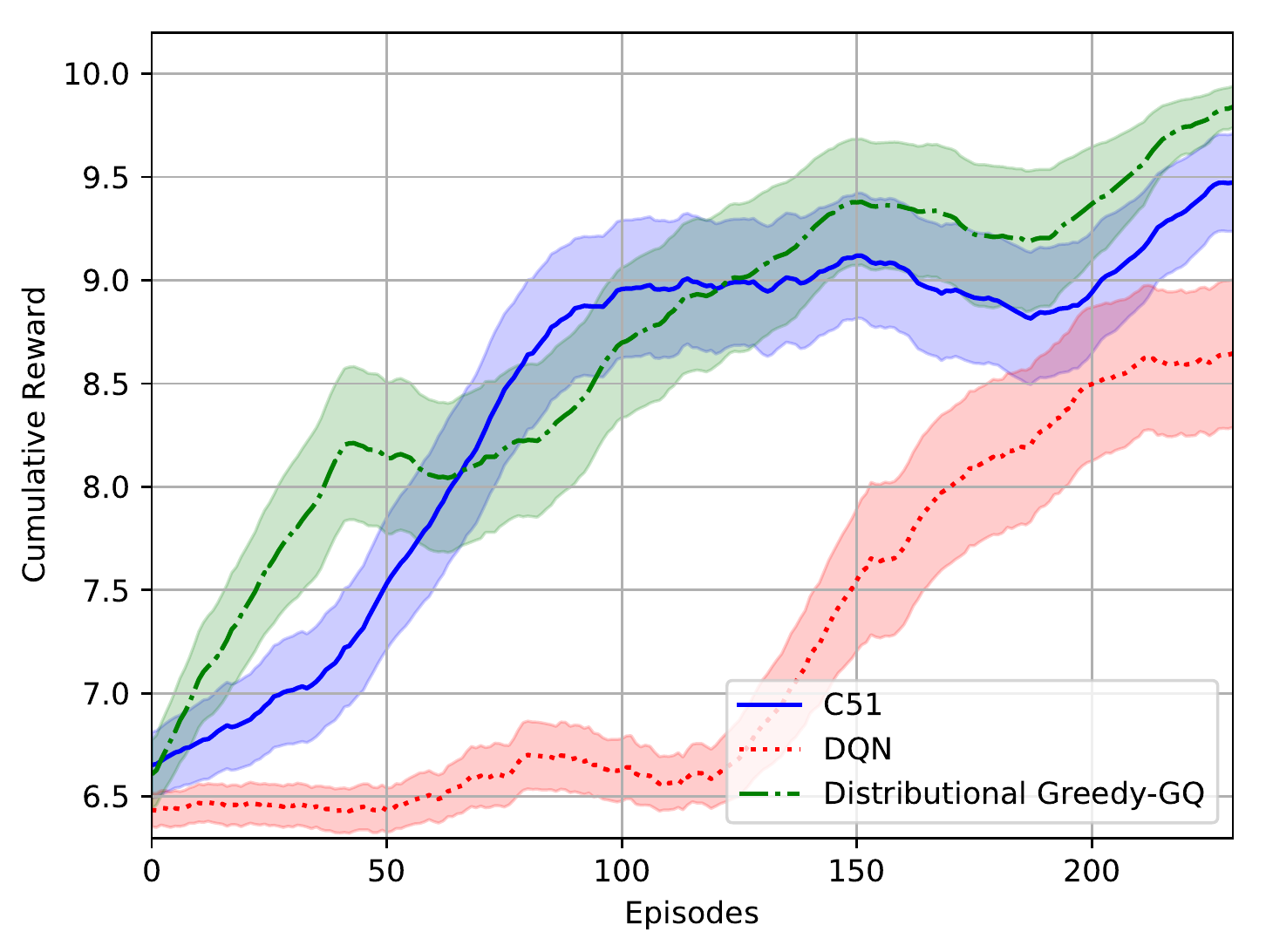}
	\includegraphics[width=0.32\textwidth]{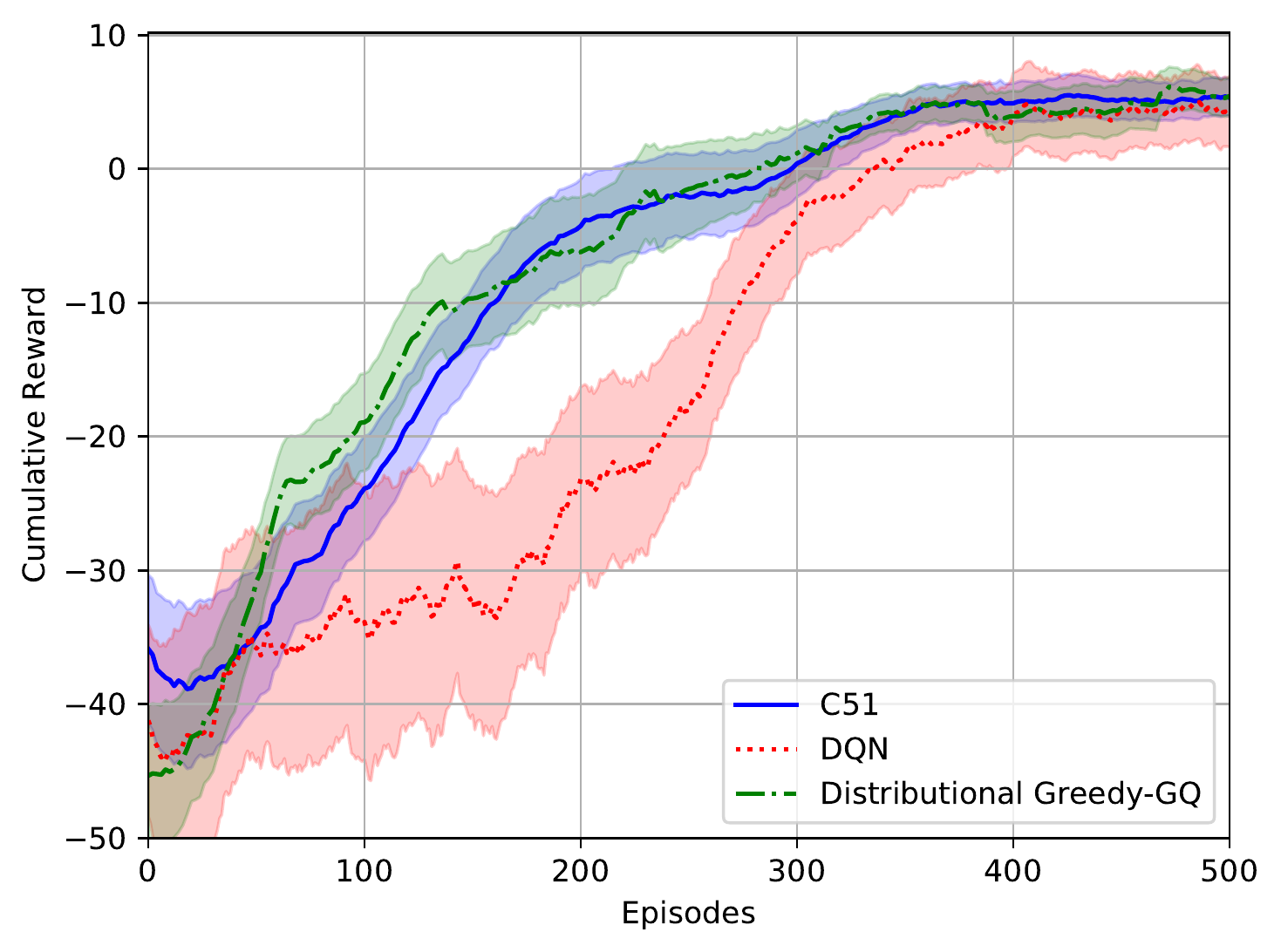}
	\includegraphics[width=0.32\textwidth]{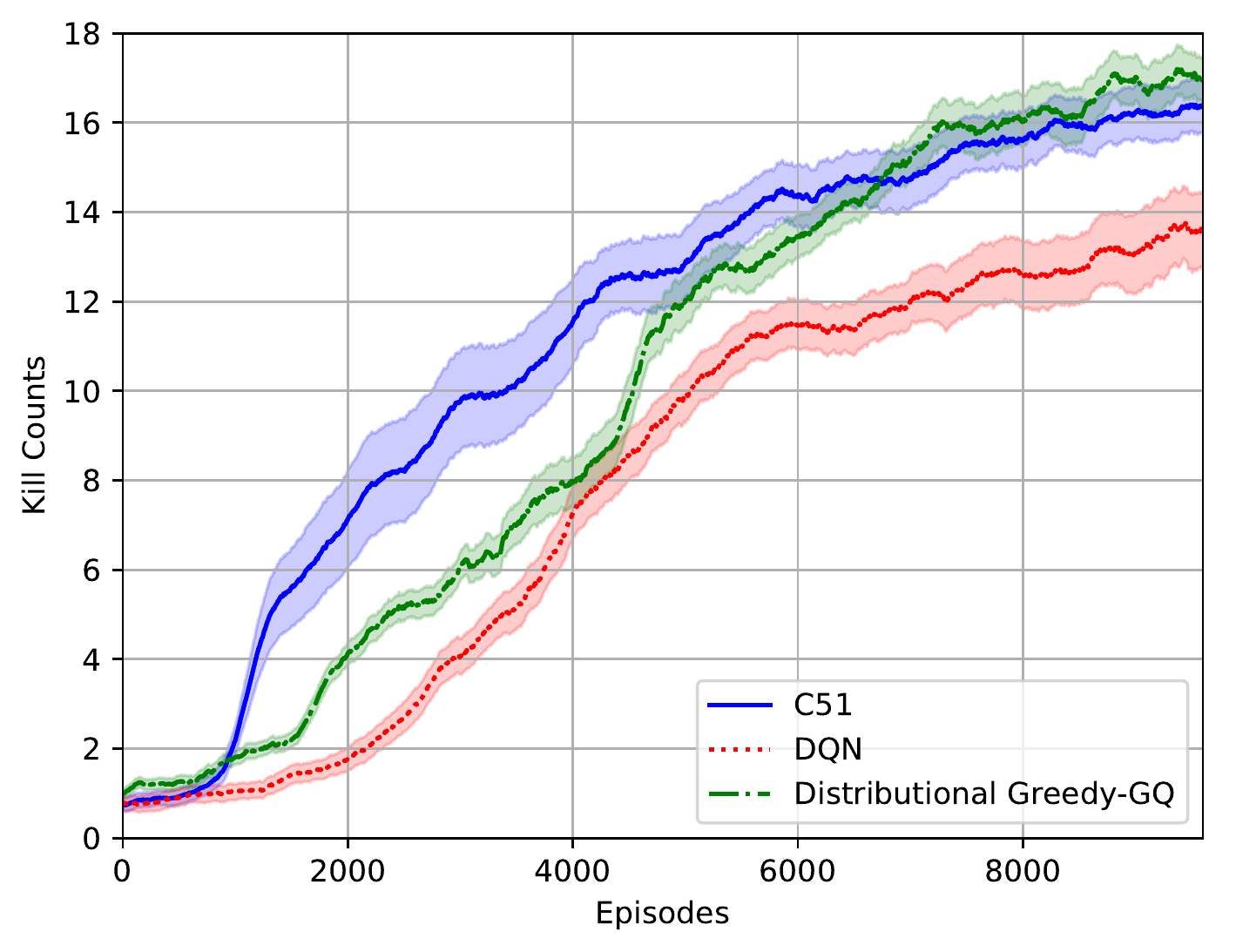}
	\caption{Left: Result in Cartpole v0. Middle: Result in lunarlander V2. Right: Result in vizdoom. In the left and middle panels, the x-axis is the training episode; the y-axis is the cumulative reward.  }\label{Fig:greedy_GQ}
\end{figure*}
\subsection{Distributional Greedy-GQ}

In practice, we are more interested in the control problem. Therefore, the aim of this section is to test the performance of distributional Greedy-GQ and compare the result with  DQN and distributional DQN (C51). All algorithm are implemented with off-policy setting with the standard experience replay.  Particularly, we test the algorithm in the environment  Cartpole v0, lunarlander v2 in the openai gym \cite{brockman2016openai} and Viz-doom \cite{kempka2016vizdoom}. In the platform of vizdoom, we choose the defend the center as a environment, where the agent occupies the center of a circular arena. Enemies continuously got spawned from far away and move closer to the agent until they are close enough to attack form close. The death penalty $-1$ is given by the environment. We give the penalty $-0.1$ if the agent loses ammo and health. Reward $+1$ is received if one enemy is killed. Totally, there are 26 bullets. The aim of the agent is to kill enemy and avoid being attacked and killed. For the  environment Cartpole and lunarlander, to implement distributional Greedy-GQ and C51, we  use two hidden-layer neural network with 64 hidden units to approximated the value distribution where activation functions are relu. The outputs are softmax functions with 40 units to approximate the probability atoms. We apply Adam with learning rate 5e-4 to train the agent. In vizdoom experiment, the first three layers are CNN  and then follows a dense layer where all activation functions are relu. The outputs are softmax functions with 50 units.  We set $V_{\min}=-10$ and $V_{\max}=20$ in the experiment.

We demonstrate all experiment in figure \ref{Fig:greedy_GQ}. In the experiment of Cartpole and vizdoom, the performance of distributional Greedy-GQ are comparable with C51 and both of them are better than the DQN. Particularly, in the left panel, distributional Greedy-GQ is slightly better than C51. The variance of them are both smaller than that of DQN possibly because the distributional algorithms are more stable. In the experiment of vizdoom, C51 learns faster than distributional Greedy-GQ at beginning but after 7000 episodes training distributional Greedy-GQ has same performance  with C51 and starts to outperform C51 later. In middle panel, C51 and distributional Greedy-GQ are slightly betten than its non-distributional counterpart DQN.

\section{Conclusion and Future work}

In this paper, we propose two non-linear distributional gradient TD algorithms and prove their convergences to a local optimum, while in the control setting, we propose the distributional Greedy-GQ.  We  compare the performance of our algorithm with their non-distributional counterparts, and show their superiority. Distributional RL has several advantages over the regular approach, e.g., it provides richer set of prediction, and the learning is more stable.
Based on this work on distributional RL, we foresee   many interesting future research directions about performing RL beyond point of the estimation of the value function. An immediate interesting one is to develop efficient exploration to devise the control policy using more distribution information rather than using the expectation.

\bibliography{DGTD}
\bibliographystyle{plainnat}

\onecolumn
\newpage
\appendix
\title{\bf Appendix: Nonlinear Distributional Gradient Temporal-Difference Learning }
\date{}
\author{}
\maketitle

\section{Proof of Theorem \ref{theorem:gradient}}

We take the gradient of $J(\theta)$ w.r.t. $\theta$ and denote $\partial_i=\frac{\partial}{\partial \theta_i}$. Notice that  $\phi_\theta(s,z_j)$ is a function depending on $\theta$ rather than the constant feature vector in the linear function approximation.
\begin{equation}
\begin{split}
-\frac{1}{2}\partial_i J(\theta)=&-\mathbb{E}\partial_i (\sum_{j=1}^{m} \phi_{\theta}(s,z_j)\delta_\theta (s,z_j) )^T A^{-1}  \mathbb{E} \sum_{j=1}^{m}\big(\phi_{\theta}(s,z_j) \delta_\theta (s,z_j) \big)\\
&+\frac{1}{2}\mathbb{E} (\sum_{j=1}^{m} \phi_{\theta}(s,z_j)\delta_\theta (s,z_j))^T A^{-1}  (\partial_i A) A^{-1} \mathbb{E} \sum_{j=1}^{m}\big(\phi_{\theta}(s,z_j) \delta_\theta (s,z_j) \big).    
\end{split}
\end{equation}

To  get around the double sampling problem \citep{dann2014policy}, we follow the idea in Gradient TD and introduce a new variable $w=A^{-1}  \mathbb{E} \sum_{j=1}^{m}\big(\phi_{\theta}(s,z_j) \delta_\theta (s,z_j) \big). $ 

Notice this $w$ is the solution of the following problem $ \min_w\|\Phi w -(\hat{G}_\theta-F_\theta)\|_D^2.$
Using stochastic gradient method, we can solve above problem and obtain the update rule of $w$ at time step $t$.
\begin{equation}\label{equ:DGTD_w}
w_{t+1}=w_{t}+\beta_t\sum_{j=1}^{m} \big(-\phi_\theta^T(s_t,z_j)w_t+  \hat{G}_\theta(s_t,z_j)-F_{\theta} (s_t,z_j) \big) \phi_\theta(s_t,z_j).
\end{equation}

Replace corresponding terms by $w$, we obtain
\begin{equation*}
\begin{split}
-\frac{1}{2} \partial_i J(\theta)=& -\mathbb{E} \partial_i \big(\sum_{j=1}^{m} \phi_{\theta}(s,z_j)\delta_\theta (s,z_j) \big)^T w+  w^T(\partial_i A)^{-1} w\\
=&-\mathbb{E} \sum_{j=1}^{m}\partial_i \delta_\theta (s,z_j)\phi_\theta(s,z_j)^Tw-\mathbb{E} \sum_{j=1}^{m}( \delta_\theta (s,z_j) \partial_i\phi_\theta(s,z_j))^Tw
+\sum_{j=1}^{m} [\phi(s,z_j)^T w (\partial_i \phi_{\theta}^T(s,z_j)) w ]\\
\end{split}
\end{equation*}
Thus we have
\begin{equation}\label{equ:DGTD_theta}
\begin{split}
-\frac{1}{2} \nabla_\theta J(\theta)=&-\mathbb{E} \sum_{j=1}^{m}\nabla_{\theta} \delta_\theta (s,z_j)\phi_\theta(s,z_j)^Tw-\mathbb{E} \sum_{j=1}^{m} (\delta_\theta (s,z_j)-w^T \phi_\theta(s,z_j))\nabla^2_{\theta} F_\theta (s,z_j)w\\ 
=& \mathbb{E} \sum_{j=1}^{m}\big(\phi_\theta(s,z_j)-\phi_\theta(s',\frac{z_j-r}{\gamma} )\big) \phi^T_\theta(s,z_j)w-\mathbb{E} \sum_{j=1}^{m} (\delta_\theta (s,z_j)-w^T \phi_\theta(s,z_j))\nabla^2_{\theta} F_\theta (s,z_j)w.
\end{split}
\end{equation} 
Observe that  
\begin{equation*}
\begin{split}
&\mathbb{E} \sum_{j=1}^{m}\nabla_{\theta} \delta_\theta (s,z_j)\phi_\theta(s,z_j)^Tw\\
= &\mathbb{E} \sum_{j=1}^{m}\big(  \phi_\theta(s',\frac{z_j-r}{\gamma} )-  \phi_\theta(s,z_j)\big) \phi^T_\theta(s,z_j)A^{-1}  \mathbb{E} \sum_{j=1}^{m}\big(\phi_{\theta}(s,z_j) \delta_\theta (s,z_j) \big)\\
=&\mathbb{E} \sum_{j=1}^{m}  \phi_\theta(s',\frac{z_j-r}{\gamma} )\phi^T_{\theta}(s,z_j)w-\mathbb{E} \sum_{j=1}^{m} (\phi_{\theta}(s,z_j)\delta_\theta(s,z_j)).
\end{split}
\end{equation*}
We can rewrite $-\nabla_{\theta} J(\theta)$ as follows
\begin{equation}\label{equ:DTDC_theta}
\begin{split}
-\frac{1}{2} \nabla_\theta J(\theta)=-\mathbb{E} \sum_{j=1}^{m} & \phi_\theta(s',\frac{z_j-r}{\gamma} )\phi_{\theta}^T(s,z_j)w+\mathbb{E} \sum_{j=1}^{m} (\phi_{\theta}(s,z_j)\delta_\theta(s,z_j))\\
&-\mathbb{E} \sum_{j=1}^{m} (\delta_\theta (s,z_j)-w^T \phi_\theta(s,z_j))\nabla^2_{\theta} F_\theta (s,z_j)w.
\end{split}
\end{equation}

\section{Proof of Theorem \ref{Theorem} }

We first prove the result on distributional GTD2.
We can write the update rule of the distributional GTD2 as 
\begin{equation}
\begin{split}
&w_{t+1}=w_{t}+\beta_t (f(\theta_t,w_t)+M_{t+1})\\
&\theta_{t+1}=\Gamma (\theta_{t} +\alpha_t (g(\theta_t,w_t)+N_{t+1} ) ),
\end{split}
\end{equation}

where $$f(\theta_t,w_t)= \mathbb{E}[ \sum_{j=1}^{m}\big(-\phi_{\theta_t}^T(s_t,z_j)w_t+  \delta_{\theta_t}  \big) \phi_{\theta_t}(s_t,z_j) \big)|\theta_t ], $$ 
$$M_{t+1}=\sum_{j=1}^{m}\big(-\phi_{\theta_t}^T(s_t,z_j)w_t+  \delta_{\theta_t} \big) \phi_{\theta_t} \big)-f(\theta_t,w_t).$$
$$ g(\theta_t,w_t)=\mathbb{E}[ \sum_{j=1}^{m}\big(\phi_{\theta_t}(s_t,z_j)-\phi_{\theta_t}(s_{t+1},\frac{z_j-r_t}{\gamma} )\big) \phi^T_{\theta_t}(s_t,z_j)w_t-h_t | \theta_t,w_t ] $$  
$$N_{t+1}=\sum_{j=1}^{m}\big(\phi_{\theta_t}(s_t,z_j)-\phi_{\theta_t}(s_{t+1},\frac{z_j-r_t}{\gamma} )\big) \phi^T_{\theta_t}(s_t,z_j)w_t-h_t-g(\theta_t,w_t).$$

We need to verify that there exists a compact set $B\subset \mathbb{R}^{2d}$ such that following four conditions hold.

\begin{enumerate}
	\item $f(\theta,w)$ and $g(\theta,w)$ are lipschitz continuous over $B$.
	\item $ \{M_t, \mathcal{G}_t\}$ and $ \{ N_t, \mathcal{G}_t\}$ are martingale difference sequences, where $\mathcal{G}_t=\sigma(r_i,\theta_i,w_i,s_i, i\leq t; s'_i,i<t)$ are increasing sigma fields.
	\item $\{ (w_t(\theta), \theta) \}$ with $w_t(\theta)$ obtained as $w_{t+1}=w_{t}+\beta_t\sum_{j=1}^{m} \big(-\phi_{\theta_t}^T(s_t,z_j)w_t+  \delta_{\theta_t}  \big) \phi_{\theta_t}(s_t,z_j) $ almost surely stays in $B$ for any choice of $(w_0(\theta),\theta)$ in $B$.
	\item ${(w,\theta_t)}$ almost surely stays in $B$ for any choice of $ (w,\theta_0)\in B .$	 
\end{enumerate}

If above four conditions are satisfied and step size $\alpha_t$, $\beta_t$ are chosen according to our assumption, then $\theta_t$ converges almost surely to the set of asymptotically stable equilibria of $\dot{\theta}=\hat{\Gamma} g(\theta, w(\theta))$ using the standard argument in \citep{sutton2009fast}. 
$w(\theta)$ is the equilibrium of 
$$\dot{w}=\mathbb{E} \sum_{j=1}^{m}\big(-\phi_{\theta_t}^T(s_t,z_j)w_t+  \delta_{\theta_t}  \big) \phi_{\theta_t}(s_t,z_j) \big). $$
Thus
$$w(\theta)=(\mathbb{E} \sum_{j=1}^{m}\phi_\theta(s,z_j) \phi_\theta^T(s,z_j))^{-1}\mathbb{E} \sum_{j=1}^{m}\big(\phi_{\theta}(s,z_j) \delta_\theta (s,z_j) \big)^T $$ which exists by our assumption.

Thus we have $g(\theta,w(\theta))=-\frac{1}{2} \nabla_\theta J(\theta).$

The condition 1 holds  from our assumption $F_{\theta}$ is three times continuously differentiable. Condition 2 holds naturally because of the way to construct $M_t$ and $N_t$. As for condition 3, because $w_t (\theta)$ converges to $w(\theta)$ using the standard argument in \citep{borkar2000ode}. Since $\theta$ in a bounded set $C$, thus $w_t(\theta)$.  Condition 4 is satisfied since $C$ is bounded.

The analysis on distributional TDC is in a similar manner. Here we just need to define a different $g(\theta_t, w_t)$ and $N_t$.

$$g(\theta_t,w_t)=\mathbb{E}[ \sum_{j=1}^{m} \big( \delta_{\theta_t}\phi(\theta_t,z_j) -\phi_{\theta_t}(s_{t+1},\frac{z_j-r_t}{\gamma})(\phi_{\theta_t}^T(s_t,z_j) w_t)\big)   -h_t|\theta_t,w_t],$$

and 

$$ N_{t+1}=\sum_{j=1}^{m} \big( \delta_{\theta_t}\phi(\theta_t,z_j) -\phi_{\theta_t}(s_{t+1},\frac{z_j-r_t}{\gamma})(\phi_{\theta_t}^T(s_t,z_j) w_t)\big)   -h_t-g(\theta_t,w_t).$$

We need to verify condition 1 and 2 again. $N_t$ is a martingale difference by its construction. Condition 1 holds using the assumption $F_\theta$ is three times continuously differentiable.

\section{The convergence result with linear function approximation}

In this section, we give a finite sample analysis when the distribution function is approximated by the linear function. Particularly, we assume $F_\theta(s,z)=\phi(s,z)^T\theta.$ 

Now the D-MSPBE reduces to 
\begin{equation*}
\begin{split}
&\frac{1}{2}J(\theta)=\frac{1}{2}\mathbb{E} \big(\sum_{j=1}^{m}\phi(s,z_j) (F_\theta(s,z_j)-\hat{G}_\theta(s,z_j))\big)^T
\big( \mathbb{E} \sum_{j=1}^{m}\phi(s,z_j) \phi^T(s,z_j) \big)^{-1}\\ & \big(\mathbb{E} \sum_{j=1}^{m}\phi(s,z_j)
(F_\theta(s,z_j)-\hat{G}_\theta (s,z_j)) \big).
\end{split}
\end{equation*}

Define $C=\mathbb{E} \sum_{j=1}^{m}\phi(s,z_j) \phi^T(s,z_j)$, $A=\mathbb{E}\sum_{i=1}^{m} \phi(s,z_j) (\phi^T(s,z_j)-\phi(s,\frac{z_j-r}{\gamma}) )$

Thus D-MSPBE is $$ \min_{\theta}\frac{1}{2}\|A\theta\|^2_{C^{-1}} $$

Using the knowledge on the convex conjugate function, we have
\begin{equation}\label{equ:saddle}
\min_{\theta} \frac{1}{2}\|A\theta\|^2_{C^{-1}}=\min_{\theta}\max_{w} L(\theta,w):= -\langle A\theta,w \rangle-\frac{1}{2}\|w\|^2_C.
\end{equation}

To solve this primal-dual problem, we apply Mirror-Prox algorithm in \cite{juditsky2011solving}.

Now the update rule of distributional GTD2 reduces to

\begin{equation}\label{equ:update_linear}
\begin{split}
&w_{t+1}=\Pi_{W}( w_{t}+\alpha_t{\sum_{j=1}^{m}} \big(-\phi^T(s_t,z_j)w_t+ { \delta_{\theta_t}}  \big) \phi(s_t,z_j))\\
&\theta_{t+1}= \Pi_{\Theta}(
\theta_t+\alpha_t {\sum_{j=1}^{m}}\big(\phi_{\theta_t}(s_t,z_j)
-{\phi_{\theta_t}(s_{t+1},\frac{z_j-r_t}{\gamma}} )\big) \phi^T_{\theta_t}(s_t,z_j)w_t),
\end{split}
\end{equation}
where $\phi$ is a constant feature vector, $\Pi_{W}$ and $\Pi_{\Theta}$ denote the projection on the compact set $W$ and $\Theta$.

The output of the algorithm is $\bar{\theta}_n:=\frac{\sum_{t=1}^{n}\alpha_t \theta_t}{\sum_{t=1}^{n}\alpha_t}$, $\bar{w}_n:=\frac{\sum_{t=1}^{n}\alpha_t y_t}{\sum_{t=1}^{n} a_t}.$

\begin{definition}
	The error function of a saddle point problem $\min_\theta \max_w L(\theta,w)=\langle b-A\theta,w \rangle+F(\theta)-K(w)$ at each point $(\theta,w)$ is defined as
	$Err(\theta', w')=\max_w L(\theta',w)-\min_\theta L(\theta, w')$
\end{definition}

\begin{assumption}\label{assumption:linear}
	1. We assume $ \Theta, W$ are compact,  The saddle point $(\theta^*, w^*)$ is in the set $\Theta\times W$. We define $D_\theta=[\max_{\theta\in \Theta} \|\theta\|_2^2-\min_{\theta\in \Theta} \|\theta\|_2^2]^{1/2}$ and $D_w=[\max_{w\in W} \|w\|_2^2-\min_{w\in W} \|w\|_2^2]^{1/2}$.  We also define $R=\max\{ \max_{\theta\in \Theta}\|\theta\|, \max_{w\in W} \|w\|. \}$
	
	2. We assume that $C=\mathbb{E} \sum_{j=1}^{m}\phi(s,z_j) \phi^T(s,z_j)$, $A=\mathbb{E}\sum_{i=1}^{m} \phi(s,z_j) (\phi^T(s,z_j)-\phi(s,\frac{z_j-r}{\gamma}) )$ are not singular.\\
	3. We assume the features $\phi_i,\phi_i'$ have uniformly bounded second moments.
\end{assumption}

We denote$\hat{A}_t$ and $\hat{C}_t$ are unbiased estimator of $A$ and $C$. We assume the variance of stochastic gradient (see our proof in the following) $\hat{A}_tw_t$ is bounded by $\sigma_1^2$ and the variance of $ \hat{A}_t\theta_t+\hat{C}w_t$ is boundd by $\sigma_2^2$. Then we have following proposition.
\begin{proposition}
	
	Suppose assumption \ref{assumption:linear} is satisfied,	let $(\bar{\theta}_n, \bar{w}_n)$ be the output of the distributional GTD2  after n iterations, Then, with probability at least $1-\delta$, we have
	
	$$Err(\bar{\theta}_n,\bar{y}_n)\leq \sqrt{\frac{10}{n}} (8+2\log (2/\delta))R\sqrt{ \sigma_1^2+\sigma_2^2+3\|A\|_2^2 R^2+2\sigma_{\max}(C) R^2}.$$
	
	Thus the convergence rate is $\mathcal{O} (\sqrt{\frac{1}{n}})$.
\end{proposition}

\begin{proof}
	The proof is easy where we need to verify the condition in the proposition 3.2 of \cite{juditsky2011solving}, thus we just list the main step here. In the following, for simplicity, we assume the reward function $R(s,a)$ is a deterministic function of $(s,a)$.
	Define the stochastic gradient vector $M$ used in the the \cite{juditsky2011solving}. 
	\begin{equation}
	M_t=\begin{bmatrix}
	M_{\theta}(\theta_t, w_t)\\
	-M_{w}(\theta_t,w_T)
	\end{bmatrix}=
	\begin{bmatrix}
	-\hat{A}_t^Tw_t\\
	\hat{A}_t\theta_t+\hat{C}_tw_t
	\end{bmatrix}
	\end{equation}	
	
	where $\hat{A}_t$ and $\hat{C}_t$ are unbiased estimator of $A$ and $C$.
	
	We apply mirror decent on the problem \eqref{equ:saddle}, and have the 
	update rule on $(\theta,w)$.
	
	\begin{equation}
	\begin{bmatrix}
	\theta_{t+1}\\
	w_{t+1}
	\end{bmatrix}=
	\Pi_{\Theta\times W}
	(\begin{bmatrix}
	\theta_{t}\\
	w_t
	\end{bmatrix}-\alpha_t M_t)
	\end{equation}
	
	It is obviously the update rule in \eqref{equ:update_linear}. Next steps are just to verify the assumptions in \cite{juditsky2011solving}
	
	We consider the distance generating function $d_{\theta}=\frac{1}{2}\|\theta\|_2^2$, $d_{w}=\frac{1}{2}\|w\|_2^2$. Then the distance generating function in $\Theta\times W$ is $ \frac{d_{\theta}}{2D^2_\theta}+\frac{w_\theta}{2D^2_{\theta}}.$

	Next step is to bound second momentum of the stochastic gradient.
	
	$$ \mathbb{E}\|-\hat{A}_t^Tw_t\|_2^2\leq \mathbb{E} \|\hat{A}_t^Tw_t-A^Tw_t\|_2^2+\|A^Tw_t\|_2^2\leq \sigma_1^2+\|A^T w_t\|_2^2\leq \sigma_1^2+\|A\|_2^2R^2  $$ 
	
	Similarly we have

	$$ \mathbb{E}\|\hat{A}_t\theta_t+\hat{C}_tw_t\|_2^2\leq \sigma_2^2 +(\|A\|_2+\sigma_{\max} (C))^2R^2 $$
	
	We denote $M^2_{*,\theta}=\sigma_1^2+\|A\|_2^2R^2$ and $M^2_{*,w}=\sigma_2^2 +(\|A\|_2+\sigma_{\max} (C))R^2$.
	
	$$ M^2_{*}=2D_\theta^2M^2_{*,\theta}+2D^2_wM^2_{*,w}=2R^2(M_{*,\theta}+M^2_{*,w})\leq 2R^2(\sigma_1^2+\sigma_2^2+3\|A\|^2_2R^2+2\sigma_{\max}^2(C)R^2). $$
	
	According to the proposition 3.2 in \cite{juditsky2011solving}, we can choose $a_t=\frac{2c}{M_{*} \sqrt{5n}}$ where $n$ is the number of training samples then with probability of at least $1-\delta$, we have
	
	$$Err(\bar{\theta}_n,\bar{y}_n)\leq \sqrt{\frac{10}{n}} (8+2\log (2/\delta))R\sqrt{ \sigma_1^2+\sigma_2^2+3\|A\|_2^2 R^2+2\sigma_{\max}(C) R^2}  $$

\end{proof}

\section{Proof of proposition \ref{proposition}. }

The (square root of) Cram{\'e}r distance has following properties \citep{bellemare2017cramer}. In the following we abuse the notation $\ell_2(P_{z_1}, P_{z_2})$ as $\ell_2(Z_1,Z_2)$.

\begin{itemize}
	\item A random variable $A$ is independent of $X$ and $Y$, then 
	$\ell_2 (X+A,Y+A)\leq \ell_2 (X,Y)$
	\item  Given a positive constant $c$, $\ell_2^2 (cX,cY)=c\ell_2^{2} (X,Y)$. Thus $\ell_2(cX,cY)\leq \sqrt{c} \ell_2(X,Y)  $
\end{itemize}

Recall $\bar{\ell}_2 (Z_1,Z_2)=\sup_{s,a} \ell_2(Z_1(s,a),Z_2(s,a))$ then $\bar{\ell}_2 (Z_1,Z_2)$ is a metric. It is easy to verify four requirements of the metric using the fact that $\ell_2$ is a metric \citep{bellemare2017cramer}. Here we just verify the triangle inequality
\begin{equation}
\begin{split}
\bar{\ell}_2 (Z_1,Z_2)=&\sup_{s,a} \ell_2(Z_1(s,a),Z_2(s,a))\\
\leq&\sup_{s,a}\big( \ell_2(Z_1(s,a), Y(s,a))+ \ell_2(Y(s,a), Z_2(s,a))\big) \\
\leq &  \sup_{s,a} \ell_2(Z_1(s,a), Y(s,a)) +  \sup_{s,a} \ell_2(Y(s,a), Z_2(s,a))\\
\leq & \bar{\ell}_2(Z_1,Y) +\bar{\ell}_2(Y,Z_2).
\end{split}
\end{equation}
Then we prove the $\sqrt{\gamma}$-contraction, for any $(s,a)$
\begin{equation}
\begin{split}
\ell_2(\mathcal{T} Z_1(s,a), \mathcal{T} Z_2 (s,a)) &= \ell_2 \big( R(s,a)+ \gamma P Z_1(s,a), R(s,a)+ \gamma P Z_2(s,a) \big)\\
& \leq \ell_2 \big(  \gamma  P Z_1(s,a),   \gamma P Z_2(s,a)     \big)\\
&\leq \sqrt{\gamma} \ell_2 \big(  P Z_1(s,a), P Z_2(s,a)  \big)\\
& \leq \sqrt{\gamma} \sup_{s',a'} \ell_2(Z_1(s',a'), Z_2(s',a') )\\
&\leq \sqrt{\gamma} \bar{\ell}_2 (Z_1,Z_2),
\end{split}
\end{equation}
where the first inequality uses the first property of Cram{\'e}r distance, the second inequality uses the second property, the third one holds from the definition of $P Z(s,a)$.

\section{Update rule of $\nabla^2 F_{\theta_{t}} (x,z_k) w$ in neural network}
Here we use a neural network with one hidden layer to illustrate the update rule, the extension to the general multilayer neural network is straight forward. The input of the neural network is state action pair $x$ whose $i^{th}$ element is $x_i$. The output is the softmax function $p_k:=p(x,z_k)$ which corresponds to the mass of the probability  at atom $z_k$. We can calculate $\nabla^2 p_k w$, and sum them to get $\nabla^2 F_\theta(x,z_k)w$.

Notice $w^T\nabla^2 p_k= w^T\nabla (\nabla p_k)$, thus we just need to act the operator $R:=w^T\nabla $ on the original forward and back propagation algorithm. 

The one hidden layer neural network has structure
$$a_j=\sum_j v_{ji} x_i, z_j=h(a_j), y_k=\sum_j v_{kj} z_j, p_m=\frac{\exp(y_m)}{\sum_m  \exp(y_m) },$$
where $v$ is the weight, h is the activation function.
Then we apply the operator $R$ on above equation and obtain
$$ R(a_j)=\sum_i w_{ji} x_i, R(z_j)=h'(a_j) R(a_j), R(y_k)=\sum_j v_{kj}R(z_j)+\sum_j w_{kj}z_j $$
$$ R(p_m)=\frac{\exp(y_m) R(y_m)}{\sum_m\exp(y_m)}-\frac{\exp y_m}{(\sum \exp(y_m))^2}\sum_m \exp(y_m)R(y_m).$$

The back propagation error of the neural network are

$$\delta_{mk}=\frac{\partial p_m}{\partial y_k}, \delta_{mj}=h'(a_j)  \sum_k v_{kj}\delta_{mk}.$$

Act the operator $R$ on them we have,

$$ R(\delta_{mk})=R(p_k)-2p_kR(p_k)~\text{when}~m=k, R(\delta_{mk})=-R(p_m)p_k-p_mR(p_k)~\text{when}~m\neq k. $$

$$ R(\delta_{mj})=h''(a_j)R(a_j)\sum_k v_{kj}\delta_{mk}+ h'(a_j)\sum_k w_{kj} \delta_{mk}+h'(a_j)\sum_k v_{kj}R(\delta_{mk}).$$

At last we have 
$$\frac{\partial p_m}{\partial v_{kj}}=\delta_{mk}z_j, \frac{\partial p_m}{\partial v_{ji}}=\delta_{mj} x_i.$$
Apply $R$ on them, we have

$$ R(\frac{\partial p_m}{\partial v_{kj}})=R(\delta_{mk})z_j+\delta_{mk}R(z_j), R(\frac{\partial p_m}{\partial v_{ji}})=R(\delta_{mj})x_i,$$
which are  $w^T\nabla^2 p_m.$

The computation complexity of whole process is in the same order of back propagation.

\end{document}